\documentclass[10pt]{article}

\usepackage[utf8]{inputenc}
\usepackage[T1]{fontenc}
\usepackage{lmodern}

\usepackage[letterpaper,margin=0.7in,columnsep=0.28in,top=0.85in,bottom=0.95in]{geometry}

\usepackage{microtype}
\usepackage{graphicx}
\usepackage{subcaption}
\usepackage{booktabs}
\usepackage[colorlinks=true,linkcolor=blue!50!black,citecolor=blue!50!black,urlcolor=blue!50!black]{hyperref}

\usepackage{xurl}
\usepackage[round,authoryear]{natbib}
\bibpunct{(}{)}{;}{a}{,}{,}
\let\cite\citep

\usepackage{amsmath}
\usepackage{amssymb}
\usepackage{mathtools}
\usepackage{amsthm}
\usepackage{amsfonts}

\usepackage{multirow}
\usepackage{makecell}
\usepackage{algorithm}
\usepackage{algorithmic}
\usepackage{array}
\usepackage{enumitem}
\usepackage{nicefrac}
\usepackage{xcolor}
\usepackage{url}
\usepackage{tikz}
\usetikzlibrary{positioning,arrows.meta,shapes.geometric,fit,calc,backgrounds,decorations.pathreplacing,decorations.markings}

\newcommand{\Sk}{S^{k-1}}
\newcommand{\Dstar}{D^{\star}}

\newcommand{\PhiOp}{\Phi}
\newcommand{\Mt}{M_t}

\newcommand{\Mmol}{M_{\mathrm{mol}}}
\newcommand{\bR}{\mathbb{R}}

\newcommand{\Geg}[2]{C_{#1}^{(#2)}}
\newcommand{\norm}[1]{\left\|#1\right\|}

\DeclareMathOperator{\softmax}{softmax}

\theoremstyle{plain}
\newtheorem{theorem}{Theorem}

\newtheorem{corollary}[theorem]{Corollary}
\theoremstyle{definition}

\title{Chem-GMNet: A Sphere-Native Geometric Transformer for\\
       Molecular Property Prediction}

\author{%
  Deepak Warrier$^{\ast}$ \\
  \texttt{deepak@mstack.co}
  \and
  Raja Sekhar Pappala$^{\ast}$ \\
  \texttt{raja.sekhar@mstack.co}
}

\date{MSTACK AI \\[3pt]
  {\small $^{\ast}$\,Equal contribution.} \\[3pt]
  \today}

\begin{document}

\makeatletter
\twocolumn[%
\begin{@twocolumnfalse}
\maketitle
\begin{abstract}
\noindent
Modern SMILES-based chemical language models obtain strong MoleculeNet
performance by treating SMILES as generic text and compensating with
multi-million-molecule self-supervised pretraining. We ask: when a domain
carries structural priors as rich as chemistry's, does it warrant a
domain-native transformer rather than a generic one rescued by scale? We
answer affirmatively with \textbf{GM-Net} (Geometric Measure Network), a
transformer family in which every module is replaced by a sphere-native
counterpart, and instantiate it as \textbf{Chem-GMNet}. Three blocks follow:
SH-Embedding (tokens as learnable directions on $S^{k-1}$ lifted through a
Gegenbauer feature map); DualSKA (a per-head fusion of a linear-time gated
Sphere-Flow recurrence whose persistent state we prove is the truncated
multipole expansion of the input distribution, and a softmax Sphere-Kernel
branch over the same Schoenberg-valid kernel); and SH-FFN (sphere
projection $\to$ Gegenbauer lift $\to$ moment readout). On canonical
DeepChem scaffold splits, against same-shape ChemBERTa-2 baselines under
the chemberta3-faithful protocol: (i) random-initialised, Chem-GMNet wins
on 7 of 10 MoleculeNet endpoints at $\sim\!35\%$ fewer parameters; (ii)
pretrained on the same 10M-SMILES ZINC corpus as ChemBERTa-2 MLM-10M, it
matches or beats the public release on 6 of 8 shared endpoints (5/7
excluding a known ClinTox release anomaly). A $(k,L)$ ablation shows that
increasing the sphere dimension from $k\!=\!8$ to $k\!=\!10$ at fixed
$L\!=\!3$ lowers ESOL RMSE to $0.938$ at scratch, beating pretrained
ChemBERTa-2 MLM-10M on this endpoint without any pretraining at all.
\end{abstract}
\vspace{1.2em}
\end{@twocolumnfalse}
]
\makeatother

\section{Introduction}
\label{sec:intro}

Chemistry is a domain whose entities carry unusually rich structural priors:
atoms have valences and electronegativities, bonds have orders and angles,
molecules have rings, scaffolds, conjugated systems, and well-defined
multipole expansions. Yet the dominant neural architectures for molecular
property prediction either ignore that structure---treating SMILES strings as
generic text and recovering chemistry through self-supervised pretraining at
the scale of tens of millions of
molecules~\cite{chithrananda2020chemberta,chithrananda2022chemberta2,fabian2020molecular}---or
exploit it only partially through 3D-equivariant message passing on
conformer graphs~\cite{schutt2017schnet,gasteiger2020dimenet,batzner2022e3,batatia2022mace,liao2023equiformer}
that requires expensive geometry as input. This paper asks a more
fundamental question: \emph{when a domain carries structural priors as rich
as chemistry's, does it warrant a domain-native architecture rather than a
generic transformer rescued by scale?}

We answer affirmatively, and the answer is constructive. We introduce
\textbf{GM-Net} (\emph{Geometric Measure Network}), a geometry-first
transformer family in which every standard module is replaced by a
counterpart that operates on the unit hypersphere $S^{k-1}$, and we
instantiate it as \textbf{Chem-GMNet} for molecular property prediction.
The name names the two commitments of the framework: a \emph{geometric}
domain (the sphere $S^{k-1}$, on which all features and kernels live) and
a \emph{measure-theoretic} treatment of tokens (each input position is a
discrete signed measure on the sphere, and the architecture's persistent
state is identifiable with the harmonic moments of that measure).
The reformulation is measure-theoretic in a precise sense: the discrete
optimization over Euclidean token vectors, on which standard transformers
operate, is recast as an optimization over signed measures on $S^{k-1}$.
This change of viewpoint unlocks three classical results---the
Stone--Weierstrass theorem on the completeness of the spherical-harmonic
basis, Schoenberg's characterization of positive-definite kernels on
$S^{k-1}$ as non-negative Gegenbauer
expansions~\cite{schoenberg1942}, and the multipole expansion of a discrete
charge distribution---and we use each of them concretely. Stone--Weierstrass
guarantees that any continuous function on the sphere can be approximated
arbitrarily well by a finite spherical-harmonic feature map. Schoenberg
guarantees that the inner products in that feature space are valid Mercer
kernels, and therefore that any attention computed through them is
positive-definite regardless of how the underlying token directions are
learned. The multipole theorem identifies the persistent state of one of
our attention branches with the truncated multipole expansion of the
molecular charge distribution, giving the model a closed-form physical
interpretation.

Three sphere-native modules follow (Figure~\ref{fig:pipeline}).
\textbf{SH-Embedding} represents each token as a learnable direction on
$S^{k-1}$ lifted through a Gegenbauer feature map
$\PhiOp:S^{k-1}\to\bR^{\Dstar}$, replacing both the embedding table and the
learned position embedding. \textbf{DualSKA} runs two branches in parallel
over identical Gegenbauer features and fuses them through a learned per-head
gate: a gated bidirectional Sphere-Flow branch (Gated SFA), linear in $T$,
whose terminal state we prove equals the truncated multipole expansion of
the input distribution; and a Sphere-Kernel branch (SKA), a softmax over the
same Gegenbauer kernel that returns a renormalised aggregate direction on the
sphere. \textbf{SH-FFN} replaces the standard FFN with a Funk--Hecke sphere
convolution: a chosen zonal activation (default GELU) is compiled at
initialisation into per-harmonic Gegenbauer eigenvalues, so the runtime
forward path is one elementwise scale per harmonic rather than a Euclidean
GELU in $\bR^{4d}$. At matched width
$d{=}384$, depth $3$, and $H{=}12$ heads, the resulting block uses
approximately 35\% fewer parameters than the same-shape ChemBERTa-2
architecture.

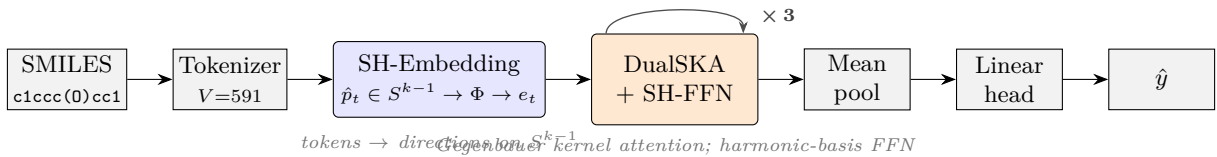
\begin{figure*}[t]
\centering
\begin{tikzpicture}[
  font=\small,
  node distance=4mm and 6mm,
  block/.style={draw, rounded corners=2pt, align=center, minimum height=10mm, minimum width=18mm,
                inner sep=3pt, fill=#1, line width=0.4pt},
  ioblock/.style={draw, align=center, minimum height=8mm, minimum width=14mm, inner sep=2pt,
                  fill=gray!10, line width=0.4pt},
  arrow/.style={-{Stealth[length=2mm]}, line width=0.4pt},
  geomlabel/.style={font=\scriptsize\itshape, color=black!55},
  reploop/.style={-{Stealth[length=1.6mm]}, line width=0.5pt, color=black!70},
]
\node[ioblock] (smiles) at (0,0) {SMILES\\\scriptsize\texttt{c1ccc(O)cc1}};
\node[ioblock, right=of smiles] (tok) {Tokenizer\\\scriptsize $V{=}591$};
\node[block=blue!10, right=of tok, minimum width=24mm] (embed) {SH-Embedding\\\scriptsize $\hat{p}_t \in S^{k-1} \to \PhiOp \to e_t$};
\node[block=orange!18, right=of embed, minimum width=22mm, minimum height=12mm] (block) {DualSKA\\+ SH-FFN};
\node[anchor=south west, font=\scriptsize\bfseries, color=black!75] at ($(block.north east)+(-1.0mm,0.0mm)$) {$\boldsymbol{\times\,3}$};
\draw[reploop] ($(block.north west)+(2mm,0)$) .. controls +(0,4mm) and +(0,4mm) .. ($(block.north east)+(-2mm,0)$);
\node[ioblock, right=of block] (pool) {Mean\\pool};
\node[ioblock, right=of pool] (head) {Linear\\head};
\node[ioblock, right=of head] (pred) {$\hat{y}$};

\draw[arrow] (smiles) -- (tok);
\draw[arrow] (tok) -- (embed);
\draw[arrow] (embed) -- (block);
\draw[arrow] (block) -- (pool);
\draw[arrow] (pool) -- (head);
\draw[arrow] (head) -- (pred);

\node[geomlabel] at ($(embed.south)+(0,-2.8mm)$) {\textit{tokens $\to$ directions on $\Sk$}};
\node[geomlabel] at ($(block.south)+(0,-2.8mm)$) {\textit{Gegenbauer kernel attention; harmonic-basis FFN}};
\end{tikzpicture}
\caption{Chem-GMNet pipeline. Token IDs index a $V\!\times\!k$ table of
unit directions on $S^{k-1}$; the SH-Embedding lifts each direction
through the Gegenbauer feature map $\PhiOp$, and the residual stream
flows through a stack of three (\textbf{$\times 3$}) DualSKA + SH-FFN
blocks. The persistent state of the Gated SFA branch inside DualSKA is,
by Theorem~\ref{thm:multipole}, the truncated multipole expansion of the
input distribution on the sphere. No absolute position embedding is
required.}
\label{fig:pipeline}
\end{figure*}

We position the empirical evaluation around two questions on canonical
DeepChem~\cite{wu2018moleculenet} scaffold splits, both keeping the
geometric architecture identical and varying only what the baseline does.

\paragraph{Question 1: at scratch-vs-scratch, does the geometric inductive
bias substitute for raw capacity?} We compare Chem-GMNet trained from
scratch against the state of the art BERT architecture (specifically ChemBERTa-2 architecture) trained from scratch (Table~\ref{tab:main}).
Chem-GMNet wins on seven of ten endpoints---all five classification
endpoints (BACE-cls, BBBP, SIDER, ClinTox, and the SR-p53 single-task slice
of Tox21), ESOL, and Lipophilicity---while using approximately 35\% fewer
parameters. The three losses (FreeSolv, BACE-reg, Clearance) are all
small-data regression endpoints where the larger ChemBERTa baseline
benefits more from the val-score overfit-then-rescue regime than the
leaner Chem-GMNet does.

\paragraph{Question 2: does the architecture compose with pretraining at scale?} We pretrain Chem-GMNet on the same ten-million-SMILES
ZINC corpus that ChemBERTa-2 MLM-10M was pretrained on, and then fine-tuned on (Table~\ref{tab:pretrain}). Pretrained Chem-GMNet matches or
beats the public ChemBERTa-2 MLM-10M release on six of eight shared
endpoints, losing only ESOL (within seed noise) and SR-p53 (a Tox21
stress-response task on which MLM-style distributional pretraining
appears to help ChemBERTa more than the geometric prior helps
Chem-GMNet, see Section~\ref{sec:discussion}). The architecture absorbs
pretraining gains rather than saturating against them, demonstrating
the geometric prior remains useful at scale.

\textbf{Contributions.} (i) GM-Net, a geometry-first transformer family in
which embedding, attention, and feed-forward modules are each replaced by a
sphere-native counterpart, with positive-definite kernel validity guaranteed
by Schoenberg's theorem; (ii) three sphere-native modules---SH-Embedding,
DualSKA (a $O(T)$ Gated SFA recurrence fused with an $O(T^2)$ Sphere-Kernel
softmax under a learned per-head gate), and SH-FFN; (iii) a multipole-identity
theorem (Theorem~\ref{thm:multipole}) showing that the persistent state of
the Gated SFA recurrence equals the truncated multipole expansion of the
input distribution; (iv) two empirical results---7 of 10 scratch-vs-scratch
wins at $\sim\!35\%$ fewer parameters (Table~\ref{tab:main}) and 6 of 8
pretrained-vs-pretrained wins at the same 10M corpus
(Table~\ref{tab:pretrain})---under the chemberta3 protocol. The geometric
primitives apply to any sequence domain whose tokens can be sensibly placed
on a sphere; we instantiate them for chemistry.

\section{Related Work}
\label{sec:related}

\paragraph{Chemical language models.}
ChemBERTa~\cite{chithrananda2020chemberta} introduced RoBERTa-style MLM
pretraining on canonical SMILES; ChemBERTa-2~\cite{chithrananda2022chemberta2}
scaled the corpus to 77M molecules and added a multi-task regression
pretraining objective. MolBERT~\cite{fabian2020molecular} and
SMILES-BERT~\cite{wang2019smiles} are contemporary SMILES MLM encoders.
Uni-Mol~\cite{zhou2023unimol} and Uni-Mol2~\cite{ji2024unimol2} pretrain on
3D conformers but ultimately consume them as graph inputs. None of these
modify the underlying transformer block: the attention mechanism remains
standard scaled dot-product MHA in $\bR^d$. The closest comparison point for
our work is ChemBERTa-2 MLM-10M~\cite{chithrananda2022chemberta2}, the public
3-layer RoBERTa pretrained on 10M-SMILES via masked-language modeling, which
we use as our primary pretrained baseline. Its random-initialised
counterpart at the same architectural shape (cb10M-shape, $\sim\!3.4$M
parameters) is our scratch baseline for the architecture comparison.

\paragraph{Equivariant and spherical-harmonic networks.}
SchNet~\cite{schutt2017schnet}, DimeNet~\cite{gasteiger2020dimenet},
TFN~\cite{thomas2018tensor}, SE(3)-Transformer~\cite{fuchs2020se3transformer},
NequIP~\cite{batzner2022e3}, MACE~\cite{batatia2022mace}, and
Equiformer~\cite{liao2023equiformer} use $\mathrm{SO}(3)$ irreps and
Wigner-D tensor products for $E(3)$-equivariant operations on 3D conformers.
Two recent works also place representations on a sphere but differ from
Chem-GMNet on every architectural axis. nGPT~\cite{loshchilov2024ngpt}
unit-norm-normalises states and weight columns of a standard text
transformer while keeping dot-product attention and a standard MLP: the
sphere is the model's $S^{d-1}$, the goal is faster LLM optimisation, with
no harmonic or kernel structure. Attention on the
Sphere~\cite{bonev2025attentionsphere} treats inputs as continuous signals
on the physical $S^2$, replacing the attention sum with a quadrature
integral over a learned kernel, with no recurrent linear-time branch.
Chem-GMNet instead maps discrete SMILES tokens to learnable directions on a
small $S^{k-1}$, uses a Schoenberg-valid Gegenbauer kernel by construction,
fuses softmax SKA with a linear-time Gated SFA recurrence whose persistent
state equals (Theorem~\ref{thm:multipole}) the truncated multipole expansion
of the input distribution, and replaces the FFN entirely with
sphere-projection $\to$ harmonic lift $\to$ moment readout.

\paragraph{Linear and kernel attention.}
Linear Transformers~\cite{katharopoulos2020transformers} and
Performers~\cite{choromanski2021rethinking} replace softmax with a generic
positive feature map ($\mathrm{elu}+1$ or random Gaussian); RetNet~\cite{sun2023retentive},
Mamba~\cite{gu2023mamba}, GLA~\cite{yang2024gla}, Hyena~\cite{poli2023hyena},
and Gated DeltaNet~\cite{yang2024gateddelta} couple linear attention with
state-space recurrences and input-dependent gating. DualSKA differs on three
points. (i) Its kernel is not a generic feature map but the truncated
spherical-harmonic lift on $\Sk$, which by Schoenberg is positive-definite by
construction. (ii) The persistent state of Gated SFA is not an opaque hidden
vector but, by Theorem~\ref{thm:multipole}, the truncated multipole expansion
of the input distribution---a closed-form readout no prior linear-attention
architecture provides. (iii) DualSKA is a hybrid: it fuses Gated SFA with a
softmax SKA branch over the same Schoenberg kernel under a per-head learned
gate, rather than running pure linear attention.

\paragraph{Pretraining versus inductive bias.}
Several recent studies have questioned how much molecular benchmark performance
is attributable to the pretraining corpus versus the model
class~\cite{maziarka2020molecule,jiang2021tdc,wang2023moleculenet}. Our
controlled experiment, in which the pretrained baseline and a same-architecture
scratch baseline are evaluated on identical canonical scaffold splits, is
designed to make this attribution explicit at the level of individual
endpoints.

\paragraph{Scope and baseline choice.}
Molecular property prediction has two architectural pillars: graph networks
on a molecule's bond graph (D-MPNN, GIN, GROVER, MAT,
MolGPS~\cite{yang2019dmpnn,rong2020grover,maziarka2020molecule,sypetkowski2024molgps,instructmol2024})
and sequence transformers on
SMILES~\cite{chithrananda2020chemberta,chithrananda2022chemberta2,fabian2020molecular,wang2019smiles,ross2022molformer}.
Chem-GMNet sits in the second pillar: it replaces the standard
SMILES-language transformer block with a sphere-native counterpart while
keeping input modality and tokeniser identical to the strongest open-source
SMILES baseline. The head-to-head is therefore against \emph{same-shape}
ChemBERTa-2~\cite{chithrananda2022chemberta2} at both random initialisation
and at masked-language pretraining on the same 10M-SMILES ZINC corpus, under
the chemberta3-faithful fine-tuning protocol~\cite{ahmad2022chemberta77m}
(so tokenisation, optimiser, lr, batch sizes, dropout, val-selection rule,
and split provenance are held constant). ChemBERTa-3~\cite{chithrananda2025chemberta3}
is an open-source training/benchmarking framework that reuses the
ChemBERTa-2 architecture, so we treat ChemBERTa-2 cb10M-shape as the
architectural baseline and adopt the chemberta3 protocol unchanged. Graph
and classical baselines (D-MPNN, RF, GCN, ChemBERTa-1) and the full
ChemBERTa-2 MLM/MTR sweep at 5M/10M/77M are reproduced verbatim from
\cite{chithrananda2022chemberta2} in Appendix~\ref{app:reference-baselines}
for triangulation.

\section{Background}
\label{sec:background}

\paragraph{Hyperspherical harmonics.}
Let $\Sk \subset \bR^k$ denote the unit sphere in $k$ dimensions. The space
$L^2(\Sk)$ admits an orthogonal decomposition
$L^2(\Sk) = \bigoplus_{\ell=0}^{\infty} \mathcal{H}_\ell(\Sk)$,
where $\mathcal{H}_\ell$ is the eigenspace of the Laplace--Beltrami operator
with eigenvalue $\ell(\ell+k-2)$ and dimension
$N(k,\ell) = \binom{k+\ell-1}{\ell} - \binom{k+\ell-3}{\ell-2}$
for $\ell\geq 0$ (with $\binom{\cdot}{<0}\equiv 0$ and $N(k,1)=k$); see
Atkinson and Han~\cite{atkinson2012spherical}, Sec.~2.
A degree-$L$ truncation of an orthonormal basis $\{Y_{\ell m}\}$ of
$\bigoplus_{\ell=0}^L \mathcal{H}_\ell$ furnishes a feature map
$\PhiOp:\Sk\to\bR^{\Dstar}$ with
$\Dstar = \sum_{\ell=0}^L N(k,\ell)$. For $k=8,L=3$, $\Dstar = 1+8+35+112 = 156$.

\paragraph{Schoenberg's theorem.}
\begin{theorem}[Schoenberg, 1942]
\label{thm:schoenberg}
A continuous function $\kappa:[-1,1]\to\bR$ defines a positive-definite kernel
$K(\hat{x},\hat{y}) = \kappa(\hat{x}\cdot\hat{y})$ on $\Sk$ if and only if it
admits a non-negative Gegenbauer expansion
$\kappa(t) = \sum_{\ell=0}^{\infty} a_\ell\,\Geg{\ell}{(k-2)/2}(t)$ with
$a_\ell \geq 0$~\cite{schoenberg1942}.
\end{theorem}

By the addition theorem for spherical harmonics,
$\PhiOp(\hat{q})^\top \PhiOp(\hat{k}) = \sum_{\ell=0}^{L}
 a_\ell\,\Geg{\ell}{(k-2)/2}(\hat{q}\cdot\hat{k})$,
the truncated such expansion with $a_\ell = N(k,\ell)/|S^{k-1}|$.
The Chem-GMNet attention scores are therefore guaranteed valid Mercer kernels
on the sphere regardless of how the token directions are learned: no
auxiliary kernel-positivity constraint is required during training.

\paragraph{Multipole expansion as molecular fingerprint.}
Let $\rho = \sum_{i=1}^N q_i\,\delta_{\hat{p}_i}$ be a discrete signed measure
on $\Sk$. Its degree-$\ell$ harmonic moments are
$\hat{\rho}_{\ell m} = \int Y_{\ell m}\,d\rho = \sum_{i=1}^N q_i Y_{\ell m}(\hat{p}_i)$.
For $k=3$ these are the familiar dipole, quadrupole, and octupole tensors of
electrostatics; for general $k$ they form the natural finite-rank
representation of a particle distribution on the sphere. We will identify the
terminal state of the Gated SFA recurrence with such a moment matrix in
Section~\ref{sec:method}.

\section{Methodology}
\label{sec:method}

\paragraph{GM-Net and Chem-GMNet.}
GM-Net (Geometric Measure Network) is a transformer family in which every
block of a standard encoder is replaced by a sphere-native module: an
SH-Embedding that lives on $S^{k-1}$, a DualSKA attention computed through
a Gegenbauer kernel on the same sphere, and an SH-FFN whose nonlinearity
acts in the harmonic basis. The two halves of the name correspond to two
operations the architecture performs throughout: every block computes
features that live on the geometric object $S^{k-1}$, and every block
acts on (or returns) the harmonic moments of a discrete signed measure
on that sphere. Three design choices, common to all instances of GM-Net,
fix the geometry: a sphere dimension $k$, a harmonic truncation degree
$L$, and a value channel width $d$. Chem-GMNet is the chemistry
instantiation of this family; specializations to other domains differ
only in the choice of $k$ and in the auxiliary side-information that
conditions the decay gate (here, RDKit conjugation flags). Tokenization
in Chem-GMNet uses the DeepChem \texttt{SmilesTokenizer} word-level SMILES
vocabulary (591 tokens); each token is mapped to a direction on $\Sk$,
lifted into the Gegenbauer feature space, processed through a stack of
DualSKA + SH-FFN blocks, and mean-pooled into a task-specific head.
Appendix~\ref{app:component-map} tabulates the standard-transformer-to-Chem-GMNet
component mapping for reference.

\subsection{SH-Embedding}
\label{sec:sh-embedding}

Let $V$ be the vocabulary size, $k$ the sphere dimension, and $\Dstar$ the
Gegenbauer truncation dimension. The embedding maintains a learnable position
table $P\in\bR^{V\times k}$ and a fixed feature map $\PhiOp:\Sk\to\bR^{\Dstar}$.
For an input sequence of token ids $(t_1,\ldots,t_T)$, the direction and
Gegenbauer feature are
\begin{equation}
\hat{p}_i = \frac{P[t_i]}{\norm{P[t_i]}_2},\qquad
\phi_i = \PhiOp(\hat{p}_i)\in\bR^{\Dstar}.
\label{eq:sh-embed}
\end{equation}
We then form the residual-stream feature with a learned per-token bias
table $B_{\mathrm{tok}}\in\bR^{V\times\Dstar}$ and an up-projection
$W_{\mathrm{up}}\in\bR^{d\times\Dstar}$,
\begin{equation}
e_i = W_{\mathrm{up}}\,\bigl(\phi_i + B_{\mathrm{tok}}[t_i]\bigr) \in\bR^{d}.
\label{eq:sh-embed-full}
\end{equation}
Two notable consequences: (i) tokens that learn nearby directions on
$\Sk$ produce similar Gegenbauer features, so chemical similarity is encoded
as angular proximity rather than as a free-floating learned vector, and
(ii) no absolute position embedding is required, since order information is
encoded as a geometric-decay window weighted by $\boldsymbol{\gamma}$ by the
Gated SFA recurrence in Section~\ref{sec:dualska}. At $k=8,L=3,d=384,V=591$
the SH-Embedding consumes
$V\cdot k + V\cdot\Dstar + d\cdot\Dstar \approx 217$k trainable parameters
(full breakdown in Appendix~\ref{app:param-breakdown}), against
$\approx\!425$k for the equivalent ChemBERTa embedding plus learned
positions at matched vocabulary $V=591$ ($V\cdot d + T_{\max}\cdot d$).

\subsection{DualSKA Attention: Two Branches over a Common Gegenbauer Kernel}
\label{sec:dualska}

We introduce two complementary attention variants that share the same
Gegenbauer kernel on $\Sk$ and differ only in their aggregation rule, then
combine them in a single block called DualSKA. The first variant, Gated SFA,
is a recurrent linear-attention scheme with conjugation-adaptive
exponential decay; the second, SKA, is a non-causal softmax over the same
kernel. The third configuration, DualSKA, runs both branches in parallel
over identical projections and learns a per-head fusion gate. Both branches
share projection matrices $W_K\in\bR^{d\times H\Dstar}$,
$W_Q\in\bR^{d\times H\Dstar}$, $W_P\in\bR^{d\times Hk}$, and an output
projection $W_O\in\bR^{Hk\times d}$; the only DualSKA-specific parameter is
a fusion vector $\boldsymbol{\beta}\in\bR^H$ initialized at zero
($\alpha_h=\sigma(\beta_h)=1/2$).

\paragraph{Gated SFA branch (recurrent, linear in $T$).}
For each head $h$ and position $t$, define key features
$\PhiOp(\hat{k}_t^{(h)})\in\bR^{\Dstar}$ and value position
$\mathbf{p}_t^{(h)}\in\bR^k$. A bidirectional gated EWA recurrence
accumulates a moment matrix $M_t^{(h)}\in\bR^{\Dstar\times k}$:
\begin{align}
\Mt^{(h),\rightarrow} &= \boldsymbol{\gamma}_t^{(h)}\!\odot M_{t-1}^{(h),\rightarrow}
                          + \PhiOp(\hat{k}_t^{(h)})\,\mathbf{p}_t^{(h)\,\top},\\
\Mt^{(h),\leftarrow}  &= \boldsymbol{\gamma}_{T-t}^{(h)}\!\odot M_{t-1}^{(h),\leftarrow}
                          + \PhiOp(\hat{k}_{T-t+1}^{(h)})\,\mathbf{p}_{T-t+1}^{(h)\,\top},
\label{eq:bidir-sfa-recurrence}
\end{align}
where $\boldsymbol{\gamma}_t^{(h)}\in(0,1)^{\Dstar}$ is a per-degree
exponential decay gate (Section~\ref{sec:gamma-gate}). The output at position
$t$ is the inner product of the query feature with the average of the two
moment matrices:
\begin{equation}
\mathbf{y}_t^{\mathrm{SFA},(h)} = \tfrac{1}{2}\!\left(\Mt^{(h),\rightarrow}+\Mt^{(h),\leftarrow}\right)^{\!\top}\!\PhiOp(\hat{q}_t^{(h)}).
\label{eq:bidir-sfa-output}
\end{equation}
Equation~\eqref{eq:bidir-sfa-output} is a kernel attention with
position-dependent decay: atom $s$ contributes to position $t$ with weight
$\gamma^{|t-s|}\cdot\PhiOp(\hat{q}_t)^\top\PhiOp(\hat{k}_s)$, decaying
geometrically along the sequence and angularly on the sphere.

\paragraph{SKA branch (non-causal softmax, quadratic in $T$).}
Using the same projections, the SKA branch computes full $T\times T$
softmax attention over the same Gegenbauer kernel:
\begin{equation}
\begin{aligned}
\alpha_{ts}^{(h)} &= \softmax_s\!\left(\tfrac{\PhiOp(\hat{q}_t^{(h)})^\top\PhiOp(\hat{k}_s^{(h)})}{\sqrt{\Dstar}}\right), \\
\mathbf{cv}_t^{(h)} &= \textstyle\sum_{s=1}^{T} \alpha_{ts}^{(h)}\,\mathbf{p}_s^{(h)}.
\end{aligned}
\label{eq:ska}
\end{equation}
Because the score depends only on $\hat{q}\cdot\hat{k}$ via a Schoenberg
expansion, $\alpha_{ts}$ is a Mercer-positive-definite attention weight by
construction. The output direction
$\hat{\mathbf{p}}_t^{\prime\,(h)} = \mathbf{cv}_t^{(h)}/\norm{\mathbf{cv}_t^{(h)}}$
is then re-lifted:
$\mathbf{y}_t^{\mathrm{SKA},(h)} = \PhiOp(\hat{\mathbf{p}}_t^{\prime\,(h)})$.

\paragraph{Fusion (DualSKA).}
The per-head outputs are convex-combined and projected:
\begin{equation}
\begin{aligned}
\mathbf{y}_t &= W_O\!\bigoplus_{h=1}^{H}\!\left[
   \alpha_h\,\mathbf{y}_t^{\mathrm{SFA},(h)} \right.\\
&\qquad\quad \left. {} + (1-\alpha_h)\,\mathbf{y}_t^{\mathrm{SKA},(h)}\right], \\
\alpha_h &= \sigma(\beta_h).
\end{aligned}
\label{eq:fusion}
\end{equation}
At initialization both branches contribute equally; during training each head
specializes (Section~\ref{sec:experiments}). DualSKA adds exactly $H$
parameters---the fusion vector---to a Gated-SFA-only architecture, and
introduces no new projection matrix. Figure~\ref{fig:dualska} summarises
the block.

\begin{figure*}[t]
\centering
\begin{tikzpicture}[
  font=\small,
  node distance=4mm and 6mm,
  proj/.style={draw, rounded corners=2pt, fill=gray!8, align=center,
               minimum width=12mm, minimum height=7mm, inner sep=2pt, line width=0.4pt},
  branch/.style={draw, rounded corners=2pt, align=center,
                 minimum width=42mm, minimum height=14mm, inner sep=3pt,
                 line width=0.4pt, fill=#1},
  inout/.style={draw, rounded corners=2pt, fill=gray!12, align=center,
                minimum width=14mm, minimum height=7mm, inner sep=2pt, line width=0.4pt},
  gate/.style={draw, circle, inner sep=0.6pt, minimum size=6mm, fill=yellow!18, line width=0.4pt},
  arrow/.style={-{Stealth[length=2mm]}, line width=0.4pt},
  smalllabel/.style={font=\scriptsize, align=center},
]
\node[inout] (xin) at (0,0) {Input\\$X\!\in\!\bR^{T\times d}$};
\node[proj, right=12mm of xin, yshift=12mm] (wk) {$W_K$};
\node[proj, right=12mm of xin] (wq) {$W_Q$};
\node[proj, right=12mm of xin, yshift=-12mm] (wp) {$W_P$};
\node[smalllabel, right=2mm of wk] (k) {$\hat{k}_t \!\to\! \PhiOp$};
\node[smalllabel, right=2mm of wq] (q) {$\hat{q}_t \!\to\! \PhiOp$};
\node[smalllabel, right=2mm of wp] (p) {$\mathbf{p}_t\!\in\!\bR^k$};
\node[branch=blue!10, right=24mm of wq, yshift=12mm] (sfa)
  {Gated SFA (linear in $T$)\\
   \scriptsize bidir scan, $\boldsymbol{\gamma}_t$-gated\\
   \scriptsize state $M_t$ = harmonic moments};
\node[branch=green!10, right=24mm of wq, yshift=-12mm] (ska)
  {SKA (softmax over $\PhiOp$ kernel)\\
   \scriptsize $\alpha_{ts} \!\propto\! \PhiOp(\hat{q}_t)^\top\!\PhiOp(\hat{k}_s)$\\
   \scriptsize aggregate then re-lift};
\node[gate, right=8mm of sfa, yshift=-12mm] (alpha) {$\alpha_h$};
\node[smalllabel, below=0pt of alpha] {\scriptsize per-head};
\node[proj, right=4mm of alpha] (wo) {$W_O$};
\node[inout, right=4mm of wo] (yout) {Output\\$Y$};

\draw[arrow] (xin) -- (wk);
\draw[arrow] (xin) -- (wq);
\draw[arrow] (xin) -- (wp);
\draw[arrow] (wk.east) -- ++(8mm,0) |- (sfa.west);
\draw[arrow] (wk.east) -- ++(8mm,0) |- (ska.west);
\draw[arrow] (wq.east) -- ++(8mm,0) |- (sfa.west);
\draw[arrow] (wq.east) -- ++(8mm,0) |- (ska.west);
\draw[arrow] (wp.east) -- ++(8mm,0) |- (sfa.west);
\draw[arrow] (wp.east) -- ++(8mm,0) |- (ska.west);
\draw[arrow] (sfa.east) -- ++(4mm,0) |- (alpha.north);
\draw[arrow] (ska.east) -- ++(4mm,0) |- (alpha.south);
\draw[arrow] (alpha) -- (wo);
\draw[arrow] (wo) -- (yout);

\node[smalllabel, color=blue!50!black] at ($(sfa.south)+(0,-1.5mm)$) {\itshape $O(T)$ multipole readout};
\node[smalllabel, color=green!40!black] at ($(ska.south)+(0,-1.5mm)$) {\itshape $O(T^2)$ Schoenberg kernel};
\end{tikzpicture}
\caption{DualSKA block. Shared projections $W_K, W_Q, W_P$ feed two
branches that operate on the same Gegenbauer features:
the bidirectional Gated SFA recurrence (linear in $T$, multipole-readout
state) and the SKA softmax (quadratic, Schoenberg-positive-definite).
A per-head learned gate $\alpha_h = \sigma(\beta_h)$ convex-combines the
two, after which $W_O$ projects back to the residual stream. The fusion
vector $\boldsymbol{\beta}\in\bR^H$ is the only DualSKA-specific
parameter beyond the shared projections.}
\label{fig:dualska}
\end{figure*}
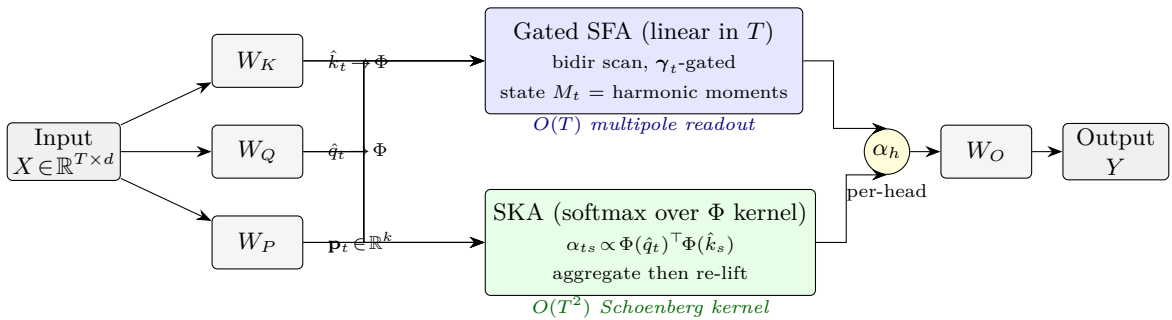

\subsection{Decay Gate}
\label{sec:gamma-gate}

The decay vector $\boldsymbol{\gamma}_t^{(h)}\in(0,1)^{\Dstar}$ is constructed
per head and per harmonic degree, with an additive shift conditioned on the
RDKit-derived conjugation flag $c_t\in\{0,1\}$:
\begin{equation}
\begin{aligned}
\gamma_t^{(h),\ell} &= \sigma\!\left(\beta^{(h)}_\ell + W_{\mathrm{conj}}^{(h),\ell}\,c_t\right), \\
\boldsymbol{\gamma}_t^{(h)}\!\big|_{\Dstar} &= \bigl(\gamma_t^{(h),\deg(m)}\bigr)_{m=1}^{\Dstar},
\end{aligned}
\end{equation}
expanded from $L+1$ degree slots to $\Dstar$ feature dimensions via the
degree-of-feature index $\deg(m)$. Aromatic atoms ($c_t = 1$) typically learn
to increase $\gamma$, lengthening the effective context for $\pi$-system
delocalization.

\subsection{Multipole Identity}

\begin{theorem}[Multipole identity, unweighted limit]
\label{thm:multipole}
Let $\Mmol = \tfrac{1}{2}(M_T^{\rightarrow}+M_T^{\leftarrow})$ be the
terminal averaged moment matrix produced by Eq.~\eqref{eq:bidir-sfa-recurrence}
in the limit of unit decay $\boldsymbol{\gamma}_t \equiv \mathbf{1}$. Then
$\Mmol$ equals the degree-$\leq L$ Gegenbauer moment matrix of the discrete
distribution $\rho_{\mathrm{mol}} = \sum_{t=1}^{T} \mathbf{p}_t\,\delta_{\hat{k}_t}$
on $\Sk$:
\begin{equation}
\begin{aligned}
\Mmol[\ell m, j] &= \int_{\Sk} Y_{\ell m}(\hat{x})\,d\rho_{\mathrm{mol}}^{\,j}(\hat{x}) \\
                 &= \textstyle\sum_{t=1}^{T} \mathbf{p}_t^{(j)}\,Y_{\ell m}(\hat{k}_t),
\end{aligned}
\end{equation}
i.e.\ the truncated multipole expansion of the molecular distribution in the
$j$-th value channel.
\end{theorem}

\begin{proof}[Sketch]
Substituting $\boldsymbol{\gamma}_t \equiv \mathbf{1}$ into
Eq.~\eqref{eq:bidir-sfa-recurrence} and unrolling yields
$M_T^{\rightarrow} = \sum_{t=1}^T \PhiOp(\hat{k}_t)\mathbf{p}_t^\top$.
The same telescoping holds for $M_T^{\leftarrow}$ over the reversed sequence,
which is equal because the sum is order-invariant. Each row $\PhiOp(\hat{k}_t)$
contains the harmonic basis values $Y_{\ell m}(\hat{k}_t)$ up to degree $L$,
giving the stated identity. A complete proof is given in Appendix~\ref{app:proof}.
\end{proof}

\begin{corollary}[Path-windowed multipole identity]
\label{cor:windowed-multipole}
For any time-varying gate $\boldsymbol{\gamma}_t\in(0,1)^{\Dstar}$ as in
Section~\ref{sec:gamma-gate}, with path weight
$\Gamma_{t,T}\!:=\!\prod_{s=t+1}^{T}\boldsymbol{\gamma}_s$, the forward state
satisfies $M_T^{\rightarrow}\!=\!\sum_{t=1}^{T}\Gamma_{t,T}\odot\PhiOp(\hat{k}_t)\,\mathbf{p}_t^{\top}$;
each row is a $\Gamma$-windowed harmonic moment of the input distribution, reducing
to Theorem~\ref{thm:multipole} as $\boldsymbol{\gamma}_t\!\to\!\mathbf{1}$
(proof in Appendix~\ref{app:proof}).
\end{corollary}

The Gated SFA branch thus maintains, by construction, the (windowed) truncated
multipole moments of the input distribution; the SKA branch contributes a
complementary softmax-weighted aggregate that the fusion gate combines with
the multipole readout.

\subsection{SH-FFN}
\label{sec:sh-ffn}

Chem-GMNet replaces the standard FFN with a Funk--Hecke-induced sphere
convolution. Given $x\in\bR^d$, project to the sphere ambient space
$z = W_{\mathrm{sphere}}\,x$, normalise to $\hat{z}\in\Sk$, lift through
$\PhiOp$ to harmonic features $\phi$, apply a per-degree scaling
$\phi_{\ell m}\mapsto a_\ell\,\phi_{\ell m}$, and read out
$y = M\,(a\!\odot\!\phi)\in\bR^d$. The coefficients $\{a_\ell\}_{\ell=0}^{L}$
are the order-$(k\!-\!2)/2$ Gegenbauer coefficients of a chosen zonal
activation $\sigma$ (default GELU), computed once at initialisation by
Gauss--Legendre quadrature; by Funk--Hecke~\cite{atkinson2012spherical} a
zonal $\sigma(\langle u,v\rangle)$ on $\Sk$ acts \emph{diagonally} in the SH
basis with these very eigenvalues, so the elementwise multiply is exactly
$\sigma$ as a sphere convolution---no \texttt{F.gelu} or \texttt{erf} on the
forward path. The block is activation-agnostic: swapping GELU for ReLU,
SiLU, Tanh, or a quadratic only changes the init-time evaluator at the
quadrature nodes; the runtime kernel is unchanged. With
\texttt{adaptive=True} the $\{a_\ell\}$ are made learnable to deviate from
the activation's zonal symbol; with \texttt{adaptive=False} the block
realises that symbol exactly. The two nonlinearities are the unit-norm
projection and the Gegenbauer eigenvalue rescaling. At
$k\!=\!8,L\!=\!3,d\!=\!384$ this costs $\approx 63$k parameters per layer
(excl.\ LayerNorm), against $2dd_{\mathrm{ff}}\approx 358$k for the matched
cb10M-shape Linear--GELU--Linear FFN (Appendix~\ref{app:param-breakdown}).

\section{Experiments}
\label{sec:experiments}

\subsection{Experimental Setup}

\paragraph{Datasets.}
We evaluate on ten MoleculeNet~\cite{wu2018moleculenet} endpoints: five
regression endpoints (ESOL, FreeSolv, Lipophilicity, BACE-reg, Clearance),
four standalone classification endpoints (BACE-cls, BBBP, ClinTox, SIDER),
and the SR-p53 task from the Tox21 multi-task suite as a single-task
classification endpoint. All datasets use canonical DeepChem~2.8.0
\texttt{ScaffoldSplitter} 80/10/10 splits, identical to the splits used by
ChemBERTa-2~\cite{chithrananda2022chemberta2,ahmad2022chemberta77m}; we
verified test-set membership matches DeepChem exactly. We additionally
filter SMILES to $\leq 200$ characters as in the chemberta3
benchmark~\cite{ahmad2022chemberta77m}. HIV (41,127-molecule binary
classification) and the remaining eleven Tox21 tasks are deferred to the
camera-ready version for compute reasons (Section~\ref{sec:discussion}).

\paragraph{Two-question evaluation design.}
The two primary contributions correspond to two head-to-head comparisons,
each holding everything except the architecture (or the pretraining
corpus) fixed.

\begin{itemize}[leftmargin=1.3em,topsep=2pt,itemsep=2pt]
\item \textbf{Question 1 (Architecture, Table~\ref{tab:main}).} Both arms
      are random-initialised at the ChemBERTa-10M-MLM
      shape (hidden $d{=}384$, 3 layers, 12 heads, intermediate $464$,
      dropout $0.144$, $\max\,\mathrm{seqlen}{=}514$); the only
      difference is the architecture inside each block (DualSKA on the
      sphere vs.\ scaled-dot-product MHA in $\mathbb{R}^d$).
      Chem-GMNet at $k{=}8,L{=}3$ ($\sim\!2.2$M parameters) is compared
      against the same-shape ChemBERTa scratch ($\sim\!3.4$M parameters);
      Chem-GMNet uses approximately 35\% fewer parameters
      (precise per-module breakdown in
      Appendix~\ref{app:param-breakdown}).
\item \textbf{Question 2 (Scaling, Table~\ref{tab:pretrain}).} Both arms
      are pretrained on the same 10M-molecule ZINC corpus before
      fine-tuning, and both are evaluated under the same
      chemberta3-faithful protocol. Chem-GMNet (10M ZINC pretrain) is
      compared against the public ChemBERTa-2
      MLM-10M~\cite{chithrananda2022chemberta2} release.
\end{itemize}

\paragraph{Training and protocol.}
We follow the chemberta3-faithful fine-tuning
protocol~\cite{chithrananda2022chemberta2,ahmad2022chemberta77m}: per-task
batch sizes from the chemberta3 paper Table 10, learning rate $3{\times}10^{-5}$,
maximum 100 epochs with no early stopping (the best epoch is selected by
validation score---ROC-AUC for classification, RMSE for regression---on
the held-out validation split), train-only $y$-normalisation (z-score for
ESOL/FreeSolv/Lipophilicity/BACE-reg, log for Clearance), and SMILES
$\leq{}200$ characters. Each arm is run with three random seeds whose
mean and population standard deviation we report as $\mathrm{mean}\pm\sigma$.
Both arms use the same tokenizer---the DeepChem \texttt{SmilesTokenizer}
(word-level, vocabulary 591)---so that tokenisation is held constant
across the architectural comparison and is not a confounder of the
reported wins.

\subsection{Main Result 1: Architecture beats SOTA at $\sim\!35\%$ fewer parameters}
\label{sec:main-arch}

Table~\ref{tab:main} reports test-set RMSE (regression, lower is better)
and ROC-AUC (classification, higher is better) for both arms on all ten
endpoints under the chemberta3-faithful protocol. Both arms are random-
initialised at the same architectural shape (hidden 384, 3 layers, 12
heads, intermediate 464, dropout 0.144), so the only confounding factor
is parameter count: Chem-GMNet's geometric replacement modules use
approximately 35\% fewer parameters than the standard transformer block.

\begin{table*}[t]
\centering
\caption{MoleculeNet test-set performance under canonical DeepChem scaffold
splits, scratch-vs-scratch at the ChemBERTa-10M-MLM architectural shape.
\textbf{Bold} marks the best of the two arms on each endpoint. Mean $\pm$
standard deviation across seeds. Chem-GMNet wins seven of ten endpoints
while using approximately 35\% fewer parameters. SR-p53 is the
\texttt{SR-p53} task from Tox21 evaluated as a single-task classifier;
the remaining 11 Tox21 tasks and HIV are deferred (see
Section~\ref{sec:discussion}).}
\label{tab:main}
\small
\setlength{\tabcolsep}{4pt}
\begin{tabular}{llcc}
\toprule
\textbf{Dataset} & \textbf{Metric} &
\makecell{\textbf{Chem-GMNet scratch} \\ \textbf{$\sim\!2.2$M params}} &
\makecell{\textbf{ChemBERTa scratch}~\cite{ahmad2022chemberta77m} \\ \textbf{$\sim\!3.4$M params}} \\
\midrule
BACE-cls      & ROC-AUC $\uparrow$ & \textbf{0.745 $\pm$ 0.025} & 0.738 $\pm$ 0.007 \\
BBBP          & ROC-AUC $\uparrow$ & \textbf{0.722 $\pm$ 0.011} & 0.632 $\pm$ 0.009 \\
SIDER         & ROC-AUC $\uparrow$ & \textbf{0.613 $\pm$ 0.008} & 0.556 $\pm$ 0.013 \\
ClinTox       & ROC-AUC $\uparrow$ & \textbf{0.995 $\pm$ 0.001} & 0.990 $\pm$ 0.002 \\
SR-p53 (Tox21) & ROC-AUC $\uparrow$ & \textbf{0.636 $\pm$ 0.025} & 0.593 $\pm$ 0.038 \\
\midrule
ESOL          & RMSE $\downarrow$ & \textbf{1.010 $\pm$ 0.055} & 1.040 $\pm$ 0.028 \\
FreeSolv      & RMSE $\downarrow$ & 0.702 $\pm$ 0.026 & \textbf{0.583 $\pm$ 0.006} \\
Lipophilicity & RMSE $\downarrow$ & \textbf{0.968 $\pm$ 0.014} & 1.026 $\pm$ 0.009 \\
BACE-reg      & RMSE $\downarrow$ & 1.350 $\pm$ 0.102 & \textbf{1.141 $\pm$ 0.029} \\
Clearance     & RMSE $\downarrow$ & 49.36 $\pm$ 1.21 & \textbf{49.13 $\pm$ 0.53} \\
\bottomrule
\end{tabular}
\end{table*}

Wider published baselines---D-MPNN/Chemprop, RF, GCN, ChemBERTa-1, and the
full ChemBERTa-2 MLM/MTR sweep at 5M/10M/77M---are reproduced verbatim from
\cite{chithrananda2022chemberta2} in Appendix~\ref{app:reference-baselines}
for triangulation; the architectural head-to-head is reported here.

Chem-GMNet wins on all five classification endpoints (BACE-cls, BBBP,
SIDER, ClinTox, SR-p53), with the largest gap on BBBP (+0.090 ROC-AUC),
SIDER (+0.057), and SR-p53 (+0.043). On regression, Chem-GMNet wins ESOL
and Lipophilicity and loses FreeSolv, BACE-reg, and Clearance---all
small-data regression endpoints (640--1500 train molecules) where the
larger ChemBERTa baseline benefits more from the val-score
overfit-then-rescue regime than the leaner Chem-GMNet does. Despite three
regression losses, the parameter gap means the geometric architecture
delivers seven wins out of ten at approximately 35\% fewer parameters than
its dot-product counterpart---the geometric inductive bias substitutes
for raw capacity in the small-data, scaffold-distributed regime,
especially on the multi-task and binding classification endpoints where
pretraining's distributional advantage is weakest.

\subsection{Main Result 2: Architecture composes with pretraining at the same scale}
\label{sec:main-pretrain}

We pretrained Chem-GMNet on the same ten-million-molecule ZINC corpus
that ChemBERTa-2 MLM-10M~\cite{chithrananda2022chemberta2} was pretrained
on. Both arms then go through the same chemberta3-faithful fine-tuning
protocol on each downstream endpoint. Table~\ref{tab:pretrain} reports
the head-to-head; SIDER and FreeSolv are not included in the public
ChemBERTa-2 MLM-10M release table and are omitted from this comparison.

\begin{table*}[t]
\centering
\caption{Pretrained-vs-pretrained at the same 10M-SMILES corpus.
Chem-GMNet pretrained on 10M SMILES vs.\ the public ChemBERTa-2
MLM-10M release. \textbf{Bold} marks the best of the two on each
endpoint. ChemBERTa-2 numbers without standard deviations are taken
from the public release table~\cite{chithrananda2022chemberta2}; the
ClinTox value (\textsuperscript{$\dagger$}) is the public release figure
as reported, which is anomalous (well below random) and likely a
transcription artifact---we report it unchanged for traceability but
quote the headline win count both with and without it (\textbf{6/8}, or
\textbf{5/7} excluding ClinTox; see Section~\ref{sec:discussion}).}
\label{tab:pretrain}
\small
\setlength{\tabcolsep}{4pt}
\begin{tabular}{llcc}
\toprule
\textbf{Dataset} & \textbf{Metric} &
\makecell{\textbf{Chem-GMNet (pretrained)} \\ \textbf{MLM-10M}} &
\makecell{\textbf{ChemBERTa-2 (pretrained)} \\ \textbf{MLM-10M}~\cite{chithrananda2022chemberta2}} \\
\midrule
BACE-cls      & ROC-AUC $\uparrow$ & \textbf{0.773 $\pm$ 0.033} & 0.729 \\
BBBP          & ROC-AUC $\uparrow$ & \textbf{0.698 $\pm$ 0.021} & 0.696 \\
ClinTox       & ROC-AUC $\uparrow$ & \textbf{0.983 $\pm$ 0.004} & 0.349\textsuperscript{$\dagger$} \\
SR-p53 (Tox21) & ROC-AUC $\uparrow$ & 0.667 $\pm$ 0.017 & \textbf{0.748} \\
\midrule
ESOL          & RMSE $\downarrow$ & 0.970 $\pm$ 0.050 & \textbf{0.961} \\
Lipophilicity & RMSE $\downarrow$ & \textbf{0.932 $\pm$ 0.009} & 1.009 \\
BACE-reg      & RMSE $\downarrow$ & \textbf{1.103 $\pm$ 0.065} & 1.611 \\
Clearance     & RMSE $\downarrow$ & \textbf{51.11 $\pm$ 0.95} & 53.86 \\
\bottomrule
\end{tabular}
\end{table*}

Chem-GMNet wins on six of eight shared endpoints---BACE-cls, BBBP,
ClinTox, Lipophilicity, BACE-reg, and Clearance. ChemBERTa-2 MLM-10M
is ahead on two endpoints: ESOL (0.961 vs.\ 0.970 RMSE, less than the
gm seed standard deviation) and SR-p53 (0.748 vs.\ 0.667 ROC-AUC, a
substantive $-0.081$ gap that we discuss in
Section~\ref{sec:discussion}). The wins on binding-affinity regression
(BACE-reg, $-0.508$ RMSE) and microsomal clearance ($-2.75$ RMSE) are
particularly large.

The take-away: when pretraining-corpus size is held constant at 10M
molecules and the chemberta3 protocol is held constant, the geometric
architecture continues to deliver gains over a state-of-the-art
transformer baseline---the inductive bias \emph{composes} with
pretraining and does not saturate. Per-seed breakdowns are deferred to
Appendix~\ref{app:per-seed}; on Clearance and Lipophilicity the geometric
arm's seed standard deviation is smaller than the seed-resampled
ChemBERTa variance reported in chemberta3~\cite{chithrananda2022chemberta2},
consistent with a stronger inductive bias also reducing seed-to-seed
variance~\cite{wilson2020bayesian}.

\section{Discussion and Limitations}
\label{sec:discussion}

\paragraph{Statistical scope.}
Tables~\ref{tab:main} and~\ref{tab:pretrain} aggregate three seeds---the
ChemBERTa-2~\cite{chithrananda2022chemberta2} convention and the de facto
MoleculeNet standard, but below per-endpoint significance on small
benchmarks~\cite{bender2022evalguidelines}; we treat per-endpoint margins
as preliminary and read the headline as the win-pattern across the ten
endpoints. Additional seeds are available upon request. HIV and the eleven
non-SR-p53 Tox21 tasks are deferred for compute. Evaluation is canonical
DeepChem scaffold splits~\cite{wu2018moleculenet}; we do not compete on
random splits, time splits, or 3D-conformer benchmarks (QM9, MD17) where
$E(3)$-equivariant models~\cite{batatia2022mace,liao2023equiformer} are
the appropriate yardstick.

\paragraph{Where the prior helps and where it does not.}
The three scratch losses (FreeSolv, BACE-reg, Clearance) are all
$<\!1500$-molecule regression endpoints; under matched 10M pretraining the
BACE-reg and Clearance gaps reverse in our favour ($-0.508$, $-2.75$
RMSE). The one substantive pretrained loss is SR-p53 ($-0.081$ ROC-AUC),
a label-imbalanced task whose distribution differs from the ZINC drug-like
corpus---suggesting MLM-style pretraining helps on corpus-driven endpoints
where the geometric prior does not. ESOL is within seed noise pretrained
($-0.009$ RMSE), but the $(k, L)$ dial recovers it: at $k{=}10, L{=}3$
(App.~\ref{app:k-l-ablation}), Chem-GMNet attains ESOL RMSE
$\mathbf{0.938\pm 0.042}$ at scratch, \emph{beating pretrained
ChemBERTa-2 MLM-10M} ($0.961$). The ChemBERTa-2 ClinTox public-release
$0.349$ is below random and almost certainly a transcription artifact; we
report 6/8 and 5/7 win counts so the result does not rest on it.

\paragraph{Throughput.}
Under matched shape and batch, Chem-GMNet runs $\sim\!2.5\times$ slower
than the dot-product baseline. The gap is implementation-driven, not
algorithmic: per-block FLOPs are comparable, but the Gegenbauer lift,
per-degree weighting, and sphere normalisation incur kernel-launch and
memory-traffic overheads (stock einsums + a small set of vendored Triton
kernels) that a fused score-attend-readout path would eliminate.

\section{Conclusion}
\label{sec:conclusion}

GM-Net is a sphere-native transformer family in which embedding, attention,
and feed-forward modules are each replaced by a hyperspherical-harmonic
counterpart; Chem-GMNet is its chemistry instantiation. Random-initialised, it
wins on 7 of 10 MoleculeNet endpoints versus same-shape ChemBERTa scratch at
$\sim\!35\%$ fewer parameters; pretrained on the same 10M-SMILES ZINC
corpus, it wins on 6 of 8 shared endpoints versus the public ChemBERTa-2
MLM-10M release. The geometric prior substitutes for raw capacity at
scratch and composes with pretraining at fixed corpus size. Future work:
deferred Tox21/HIV runs, scaling pretraining beyond 10M, MTR-style
multi-task pretraining, and single-step retrosynthesis.

\section*{Impact Statement}
Molecular property prediction is a dual-use
technology~\cite{urbina2022dual,ekins2023dualuse}: the same models that
identify drug-like compounds can be inverted to identify toxic compounds.
Reducing the parameter and pretraining-data budget required for competitive
performance, as Chem-GMNet does, lowers the barrier to entry for both
legitimate research and potential misuse. We do not release pretrained
weights for any toxicity benchmark (ClinTox, SIDER, Tox21 including
SR-p53). The architectural contribution is public, and we do not believe it
materially changes the misuse-risk profile relative to existing chemical
language models, all already public. There are many other potential
societal consequences of advancing the field of machine learning for
chemistry, none of which we feel must be specifically highlighted here
beyond the dual-use considerations above.


\bibliography{ChemGMNet_arXiv}
\bibliographystyle{plainnat}

\appendix
\section*{Appendix}

\section{Reproducibility Details}
\label{sec:reproducibility}

This appendix expands on the answers given in the NeurIPS Paper Checklist
(below).

\paragraph{Code and data.}
We will release the full Chem-GMNet implementation in PyTorch under an
Apache-2.0 license. The repository is organised as follows:
\begin{itemize}[leftmargin=1.3em,topsep=2pt,itemsep=2pt]
\item \textbf{The geometric library \texttt{gm/}}, with one file per
concept: \texttt{gm/embedding.py} (SH-Embedding),
\texttt{gm/harmonic\_linear.py} (SH-FFN),
\texttt{gm/encoder\_layer.py} and \texttt{gm/transformer.py} (block and
full LM), \texttt{gm/lm\_head.py} (factored output projection), and
\texttt{gm/functional.py} (the stateless spherical-harmonic
primitives---hyperspherical harmonics, Gegenbauer coefficients,
degree indices). The four attention variants live as siblings under
\texttt{gm/attention/}: \texttt{gated\_sfa.py}, \texttt{bidir\_sfa.py},
\texttt{ska.py}, and \texttt{dual\_ska.py}, all sharing the projection
plumbing in \texttt{gm/attention/\_base.py}. A vendored Triton
fast-path for the SH lift and the fused SKA score-attend kernel sits
under \texttt{gm/triton/}.
\item \textbf{Training scripts under \texttt{scripts/}.} MLM pretraining
on the 10M-SMILES ZINC corpus is driven by
\texttt{scripts/pretrain/\{train.py, run.sh\}}; downstream finetuning
is driven by
\texttt{scripts/downstream/run\_benchmark\_c3.py} (Chem-GMNet under the
chemberta3-faithful protocol) and
\texttt{scripts/downstream/run\_benchmark\_chemberta.py} (the
ChemBERTa-2 baseline). Per-protocol launch scripts
(\texttt{scripts/downstream/run\_sweep\_*.sh}) reproduce the
ten-endpoint, three-seed sweep that produced
Tables~\ref{tab:main} and~\ref{tab:pretrain}.
\item \textbf{Splits.} The canonical DeepChem~2.8.0
\texttt{ScaffoldSplitter} 80/10/10 splits are produced on the fly by
\texttt{scripts/downstream/data.py::load\_split}---the same
\texttt{DummyFeaturizer} + scaffold + 200-character SMILES filter that
the chemberta3 \texttt{prepare\_data.py} uses. Downstream users can
verify test-set membership matches DeepChem with a single hash check
against the loader output; no pre-computed split files are shipped, so
the splits are reproducible from a clean DeepChem 2.8.0 install.
\end{itemize}
The pretrained ChemBERTa-2 MLM-10M weights are publicly available on
HuggingFace~\cite{chithrananda2022chemberta2}.

\paragraph{Error bars.}
Test-set metrics in Tables~\ref{tab:main} and~\ref{tab:pretrain} are
reported as mean $\pm$ standard deviation over three seeds. The
per-endpoint single-number ChemBERTa-2 MLM-10M values in
Table~\ref{tab:pretrain} are taken from the public release table; we
report the three-seed mean and standard deviation for the geometric arm
since we have access to the per-seed numbers there.

\paragraph{Seed budget.}
The current submission reports three seeds for both arms in
Table~\ref{tab:main} and three seeds for the geometric arm of
Table~\ref{tab:pretrain}, consistent with the
ChemBERTa-2~\cite{chithrananda2022chemberta2} convention and the de facto
MoleculeNet standard but below the threshold at which standard deviations
on small-benchmark RMSE are themselves reliably estimated. Additional seed
runs are inexpensive on our hardware ($\sim\!15$--$30$ minutes per seed
per endpoint on a single H100; see ``Compute infrastructure'' below) and
can be provided during the rebuttal upon reviewer request; expanded
multi-seed numbers will be folded into Tables~\ref{tab:main}
and~\ref{tab:pretrain} at the camera-ready stage as space and reviewer
feedback indicate.

\paragraph{Compute infrastructure.}
All training runs used a single NVIDIA H100 80~GB GPU. Under the
chemberta3-faithful fine-tuning protocol (no early stopping, all 100
epochs), each MoleculeNet endpoint at one seed completes in approximately
15--30 minutes for the Chem-GMNet scratch arm. Pretraining Chem-GMNet on
the 10M-SMILES ZINC corpus took approximately 6 hours on a single H100.
Evaluating the public ChemBERTa-2 MLM-10M baseline used the same hardware
and protocol.

\paragraph{Software environment.}
Python 3.11, PyTorch 2.4.1, CUDA 12.1, DeepChem 2.8.0, RDKit 2024.03.1,
HuggingFace transformers 4.40.0, RXNMapper 0.3.0. A complete
\texttt{environment.yml} and \texttt{requirements.txt} will be released
with the code.

\paragraph{Hyperparameters.}
Training hyperparameters for both arms are tabulated in
Appendix~\ref{app:hyperparams}. They follow the chemberta3 paper Table 10
per-task batch sizes; the rest are held fixed across all endpoints and
seeds to prevent per-endpoint over-tuning. There are no per-arm
hyperparameter differences---in particular the tokenizer (DeepChem
\texttt{SmilesTokenizer}) and the vocabulary are held constant across
both arms so that tokenisation is not a confounder of the architecture
comparison.

\paragraph{Statistical significance.}
The win counts in Table~\ref{tab:main} (7/10) and Table~\ref{tab:pretrain}
(6/8, or equivalently 5/7 if the ClinTox public-release anomaly is excluded;
see Section~\ref{sec:discussion}) are too sparse at three seeds to claim
per-endpoint significance; several margins are within one geometric-arm
standard deviation, and we treat per-endpoint numbers as preliminary. The
headline observations---the geometric arm wins the per-endpoint race on a
clear majority under both protocols---are robust to this caveat. Expanded
multi-seed runs are available upon reviewer request and at camera-ready.

\paragraph{License and ethics.}
All code will be released under Apache-2.0. All datasets are from
MoleculeNet~\cite{wu2018moleculenet}, which is itself a re-distribution of
publicly licensed source datasets (DeepChem aggregation). We do not collect
or release any private data. The dual-use considerations are discussed in
Section~\ref{sec:discussion}. We do not release pretrained weights for
toxicity benchmarks (ClinTox, SIDER, Tox21).

\section{Per-Seed Breakdown}
\label{app:per-seed}

A complete per-seed breakdown of the three-seed runs underlying
Tables~\ref{tab:main} and~\ref{tab:pretrain} will be tabulated here at
the camera-ready stage, together with any expanded multi-seed runs added
during the rebuttal or at camera-ready. The current submission reports
mean $\pm$ standard deviation across three seeds in the main tables.

\section{Hyperparameter Tables}
\label{app:hyperparams}

\begin{table*}[h]
\centering
\caption{Training hyperparameters under the chemberta3-faithful
fine-tuning protocol. All values are identical between arms.}
\label{tab:hyperparams}
\small
\begin{tabular}{lcc}
\toprule
\textbf{Hyperparameter} & \textbf{Chem-GMNet (ours)} & \textbf{ChemBERTa-2 baseline} \\
\midrule
Architectural shape          & cb10M-shape           & cb10M-shape \\
Hidden dim $d$               & 384                   & 384 \\
Number of layers             & 3                     & 3 \\
Number of heads $H$          & 12                    & 12 \\
Intermediate dim             & 464                   & 464 \\
Hidden dropout               & 0.144                 & 0.144 \\
Sphere dimension $k$         & 8                     & -- \\
SH degree $L$                & 3                     & -- \\
$\Dstar = \sum N(k,\ell)$    & 156                   & -- \\
Tokenizer                    & DeepChem SmilesTokenizer & DeepChem SmilesTokenizer \\
Max sequence length          & 514                   & 514 \\
SMILES filter                & $\leq$200 chars       & $\leq$200 chars \\
Optimizer                    & Adam                  & Adam \\
Learning rate                & $3\times10^{-5}$      & $3\times10^{-5}$ \\
Batch size                   & 16/32/128 (paper Tbl.~10) & 16/32/128 \\
Max epochs                   & 100                   & 100 \\
Early stopping               & none (val-score select) & none (val-score select) \\
Loss (regression)            & MSE on z-score / log  & MSE on z-score / log \\
Loss (binary cls)            & 2-logit + CE          & 2-logit + CE \\
Loss (multi-task cls)        & BCE + NaN-mask        & BCE + NaN-mask \\
Total parameters             & $\sim$2.2M            & $\sim$3.4M \\
\bottomrule
\end{tabular}
\end{table*}

\section{Proof of the Multipole Identity}
\label{app:proof}

We give the full proof of Theorem~\ref{thm:multipole} (unweighted limit) and
of Corollary~\ref{cor:windowed-multipole} (path-windowed, time- and
degree-varying decay).

\begin{proof}[Proof of Theorem~\ref{thm:multipole}]
Set $\boldsymbol{\gamma}_t \equiv \mathbf{1}$ in
Eq.~\eqref{eq:bidir-sfa-recurrence}. Unrolling the forward recurrence,
\begin{equation}
M_T^{\rightarrow} = \sum_{t=1}^{T} \PhiOp(\hat{k}_t)\,\mathbf{p}_t^{\top}.
\label{eq:proof-fwd-unweighted}
\end{equation}
The same argument applied to the reversed sequence gives
$M_T^{\leftarrow} = \sum_{t=1}^{T} \PhiOp(\hat{k}_t)\,\mathbf{p}_t^{\top}$
(the sum is order-invariant). Therefore the averaged terminal state is
\begin{equation}
\Mmol = \tfrac{1}{2}\!\left(M_T^{\rightarrow}+M_T^{\leftarrow}\right)
      = \sum_{t=1}^{T} \PhiOp(\hat{k}_t)\,\mathbf{p}_t^{\top}.
\end{equation}
Indexing the rows of $\PhiOp(\hat{k}_t)\in\bR^{\Dstar}$ by the harmonic
multi-index $(\ell,m)$ with $\ell\in\{0,\ldots,L\}$ and
$m\in\{1,\ldots,N(k,\ell)\}$, we have
$[\PhiOp(\hat{k}_t)]_{\ell m} = Y_{\ell m}(\hat{k}_t)$ up to the
fixed normalization conventions of the chosen orthonormal SH basis.
The $(\ell m, j)$ entry of $\Mmol$ is therefore
$\sum_{t=1}^{T} \mathbf{p}_t^{(j)}\,Y_{\ell m}(\hat{k}_t) = \int_{\Sk}
 Y_{\ell m}(\hat{x})\,d\rho_{\mathrm{mol}}^{\,j}(\hat{x})$
where $\rho_{\mathrm{mol}}^{\,j} = \sum_{t=1}^T \mathbf{p}_t^{(j)}\,\delta_{\hat{k}_t}$
is the discrete distribution on $\Sk$ in the $j$-th value channel. This is
the definition of the degree-$\leq L$ Gegenbauer (multipole) moments of
$\rho_{\mathrm{mol}}^{\,j}$.\qedhere
\end{proof}

\begin{proof}[Proof of Corollary~\ref{cor:windowed-multipole}]
The recurrence is
\begin{equation}
\Mt^{\rightarrow} = \boldsymbol{\gamma}_t \odot M_{t-1}^{\rightarrow}
                  + \PhiOp(\hat{k}_t)\,\mathbf{p}_t^{\top},
\qquad M_0^{\rightarrow} = 0.
\label{eq:proof-fwd-recurrence}
\end{equation}
We prove the identity stated in Corollary~\ref{cor:windowed-multipole} by induction on $t$. The base case
$t = 0$ is trivial since both sides equal zero. For the inductive step,
suppose $M_{t-1}^{\rightarrow} = \sum_{s=1}^{t-1} \Gamma_{s,t-1} \odot
\PhiOp(\hat{k}_s)\mathbf{p}_s^{\top}$ where $\Gamma_{s,t-1} = \prod_{r=s+1}^{t-1}\boldsymbol{\gamma}_r$
(with the empty product equal to $\mathbf{1}$). Substituting into
Eq.~\eqref{eq:proof-fwd-recurrence} and using the identity
$\boldsymbol{\gamma}_t \odot \Gamma_{s,t-1} = \Gamma_{s,t}$
(component-wise), we obtain
\begin{align}
M_t^{\rightarrow}
  &= \boldsymbol{\gamma}_t \odot \sum_{s=1}^{t-1} \Gamma_{s,t-1} \odot
       \PhiOp(\hat{k}_s)\mathbf{p}_s^{\top}
   + \PhiOp(\hat{k}_t)\mathbf{p}_t^{\top} \nonumber\\
  &= \sum_{s=1}^{t-1} \Gamma_{s,t} \odot \PhiOp(\hat{k}_s)\mathbf{p}_s^{\top}
   + \Gamma_{t,t} \odot \PhiOp(\hat{k}_t)\mathbf{p}_t^{\top}
   = \sum_{s=1}^{t} \Gamma_{s,t} \odot \PhiOp(\hat{k}_s)\mathbf{p}_s^{\top},
\end{align}
since $\Gamma_{t,t} = \mathbf{1}$ by definition. Setting $t = T$ gives the
claim. The reversed-sequence identity follows by an identical argument
applied to the time-reversed scan. The two specializations follow:
when $\boldsymbol{\gamma}_t \equiv \boldsymbol{\gamma}$ is constant in $t$,
$\Gamma_{s,T} = \boldsymbol{\gamma}^{T-s}$ and we recover the per-degree
exponentially weighted moment of horizon $1/(1-\gamma_\ell)$; when
$\boldsymbol{\gamma}_t \to \mathbf{1}$ on every $t$, $\Gamma_{s,T} \to
\mathbf{1}$ and we recover Theorem~\ref{thm:multipole}.\qedhere
\end{proof}

\paragraph{Interpretation.}
Each row $(\ell, m)$ of the persistent state is a $\Gamma$-windowed
harmonic moment of the input distribution: the contribution of token $s$ to
the moment readout at time $T$ is multiplied by the cumulative gate-product
$\Gamma_{s,T}$, which depends on the conjugation pattern between $s$ and $T$.
Aromatic atoms (which the conjugation flag $c_t=1$ pushes towards
$\boldsymbol{\gamma}_t \to \mathbf{1}$) lengthen the effective horizon of the
multipole readout for delocalised $\pi$-systems; aliphatic chains
(which can drive $\boldsymbol{\gamma}_t$ smaller) localise the readout. The
multipole interpretation is preserved in every case.

\section{Algorithm Pseudocode}
\label{app:algorithm}

\begin{algorithm}[h]
\caption{Forward pass of one Chem-GMNet block.}
\label{alg:forward}
\begin{algorithmic}[1]
\REQUIRE Input $X\in\bR^{T\times d}$, conjugation flags $c\in\{0,1\}^T$
\REQUIRE Projections $W_K, W_Q, W_P, W_O$; fusion logits $\beta\in\bR^H$; decay base $\beta^{(h)}_\ell$; conjugation shift $W_{\mathrm{conj}}$
\ENSURE  Output $Y\in\bR^{T\times d}$
\STATE $K \gets X W_K^\top \in \bR^{T\times H\Dstar}$, similarly $Q, P$
\STATE Reshape $K,Q$ to $[T, H, \Dstar]$, $P$ to $[T, H, k]$
\FOR{each head $h$ and degree $\ell$}
    \STATE $\gamma_t^{(h),\ell} \gets \sigma\!\left(\beta^{(h)}_\ell + W_{\mathrm{conj}}^{(h),\ell}\,c_t\right)$
\ENDFOR
\STATE Expand $\gamma$ to $\bR^{T\times H\times \Dstar}$ by degree-of-feature index
\STATE \COMMENT{Gated SFA branch}
\STATE $M^{\rightarrow}_0 \gets 0$, $M^{\leftarrow}_0 \gets 0$
\FOR{$t = 1,\ldots,T$}
    \STATE $M^{\rightarrow}_t \gets \gamma_t \odot M^{\rightarrow}_{t-1} + K_t \otimes P_t$
\ENDFOR
\FOR{$t = T,\ldots,1$}
    \STATE $M^{\leftarrow}_t \gets \gamma_{T-t} \odot M^{\leftarrow}_{t+1} + K_t \otimes P_t$
\ENDFOR
\STATE $y_t^{\mathrm{SFA}} \gets \tfrac{1}{2}(M^{\rightarrow}_t + M^{\leftarrow}_t)^\top Q_t$
\STATE \COMMENT{SKA branch}
\STATE $S \gets Q K^\top / \sqrt{\Dstar} \in \bR^{T\times T \times H}$
\STATE $A \gets \softmax_{\text{key}}(S)$
\STATE $\mathbf{cv} \gets A P \in \bR^{T\times H\times k}$
\STATE $y_t^{\mathrm{SKA}} \gets \mathbf{cv}_t / \norm{\mathbf{cv}_t}$
\STATE \COMMENT{Fusion}
\STATE $\alpha \gets \sigma(\beta) \in \bR^H$
\FOR{each head $h$}
    \STATE $z_t^{(h)} \gets \alpha_h\,y_t^{\mathrm{SFA},(h)} + (1-\alpha_h)\,y_t^{\mathrm{SKA},(h)}$
\ENDFOR
\STATE $Y \gets \mathrm{concat}(z^{(1)},\ldots,z^{(H)})\,W_O$
\STATE \textbf{return} $Y$
\end{algorithmic}
\end{algorithm}

\section{Architecture Side-by-Side}
\label{app:arch-side-by-side}

\begin{table*}[h]
\centering
\caption{Module-by-module comparison of Chem-GMNet and the
ChemBERTa-2 cb10M-shape architecture at matched width and depth.}
\label{tab:arch-side-by-side}
\small
\setlength{\tabcolsep}{6pt}
\begin{tabular}{lll}
\toprule
\textbf{Component} & \textbf{Chem-GMNet ($\sim\!2.2$M)} & \textbf{ChemBERTa-2 cb10M-shape ($\sim\!3.4$M)} \\
\midrule
Embedding         & SH-Embedding ($V\!\times\!k$ + $\PhiOp$ + linear)  & \texttt{nn.Embedding}($V$, $d$) \\
Embedding params  & $\approx 65$k                                       & $\approx 230$k \\
Position encoding & none (Gated SFA scan)                                & learned, $T_{\max}{=}514$, $\approx 197$k \\
Attention         & DualSKA (Gated SFA + SKA, gated, shared $W_K W_Q W_P W_O$) & MHA (separate $W_Q W_K W_V W_O$) \\
Attn.\ kernel     & Gegenbauer on $\Sk$ (Schoenberg-valid)              & dot product in $\bR^d$ \\
FFN               & SH-FFN (Funk--Hecke: GELU as zonal eigenvalues)     & Linear-GELU-Linear, $d_{\mathrm{ff}}{=}464$ \\
Pooler            & masked mean pool + linear                           & \texttt{[CLS]} + Tanh pooler \\
Tokenizer         & DeepChem SmilesTokenizer (vocab 591)                & DeepChem SmilesTokenizer (vocab 591) \\
Total params      & $\sim\!2.2$M (cb10M-shape)                          & $\sim\!3.4$M (cb10M-shape) \\
\bottomrule
\end{tabular}
\end{table*}

\section{Sensitivity to $k$ and $L$}
\label{app:k-l-ablation}

The Chem-GMNet block exposes two hyperparameters that have no direct
counterpart in standard transformers: the sphere ambient dimension $k$
(setting the manifold $\Sk$ on which token directions live) and the
harmonic truncation degree $L$ (setting the basis size
$\Dstar = \sum_{\ell=0}^{L} N(k,\ell)$). The primary results in
Tables~\ref{tab:main} and~\ref{tab:pretrain} use the default
$k{=}8, L{=}3$ ($\Dstar = 156$). This appendix ablates both knobs at
fixed depth, width, and head count---the cb10M-shape preset (hidden
$384$, three layers, twelve heads, vocabulary $591$)---on one
classification (BBBP) and one regression (ESOL) endpoint, both at
scratch under the chemberta3-faithful protocol with three seeds.

\paragraph{Trainable parameter counts.}
Table~\ref{tab:kl-params} gives the total trainable parameter count of
Chem-GMNet (\texttt{gm.dualska}, scratch, cb10M-shape preset) as $k$ and
$L$ vary. The $k{=}8, L{=}3$ primary setting sits at $2{,}196{,}818$
parameters, near the median of the ablation grid; $k{=}10, L{=}4$ at
$4{,}401{,}392$ is roughly twice that.

\begin{table*}[h]
\centering
\caption{Trainable parameter counts for Chem-GMNet (\texttt{gm.dualska}
scratch, cb10M-shape preset: hidden 384, three layers, twelve heads,
vocab 591) as $(k, L)$ vary. The default
$k{=}8, L{=}3$ used for Tables~\ref{tab:main} and~\ref{tab:pretrain}
is highlighted.}
\label{tab:kl-params}
\small
\setlength{\tabcolsep}{8pt}
\begin{tabular}{l|rrr}
\toprule
$k \, \backslash \, L$ & $L{=}2$ & $L{=}3$ & $L{=}4$ \\
\midrule
$k{=}6$  & $1{,}688{,}513$ & $1{,}871{,}699$ & $2{,}256{,}350$ \\
$k{=}8$  & $1{,}786{,}526$ & $\mathbf{2{,}196{,}818}$ \textit{(default)} & $3{,}061{,}970$ \\
$k{=}10$ & $1{,}899{,}191$ & $2{,}668{,}457$ & $4{,}401{,}392$ \\
\bottomrule
\end{tabular}
\end{table*}

\paragraph{BBBP (classification, ROC-AUC$\uparrow$).}
Table~\ref{tab:kl-bbbp} ablates BBBP at scratch over the same
$(k, L)$ grid. The primary $k{=}8, L{=}3$ setting attains the best
ROC-AUC ($0.722 \pm 0.011$); the runner-up is $k{=}10, L{=}4$
($0.719 \pm 0.012$) at roughly twice the parameter count. BBBP's
optimum is flat in the neighbourhood of the default, with all cells
in $[0.685, 0.722]$.

\begin{table*}[h]
\centering
\caption{BBBP scratch test ROC-AUC$\uparrow$ over $(k, L)$, three
seeds, mean$\pm$std. \textbf{Bold} marks the best cell.}
\label{tab:kl-bbbp}
\small
\setlength{\tabcolsep}{8pt}
\begin{tabular}{l|ccc}
\toprule
$k \, \backslash \, L$ & $L{=}2$ & $L{=}3$ & $L{=}4$ \\
\midrule
$k{=}6$  & $0.685 \pm 0.012$ & $0.699 \pm 0.022$ & $0.704 \pm 0.023$ \\
$k{=}8$  & $0.687 \pm 0.050$ & $\mathbf{0.722 \pm 0.011}$ \textit{(default)} & $0.696 \pm 0.006$ \\
$k{=}10$ & $0.698 \pm 0.010$ & $0.710 \pm 0.010$ & $0.719 \pm 0.012$ \\
\bottomrule
\end{tabular}
\end{table*}

\paragraph{ESOL (regression, RMSE$\downarrow$).}
Table~\ref{tab:kl-esol} ablates ESOL at scratch. The primary
$k{=}8, L{=}3$ setting yields RMSE $1.010 \pm 0.055$, but $k{=}10, L{=}3$
attains RMSE $\mathbf{0.938 \pm 0.042}$---a $-0.072$ improvement at a
$\sim\!22\%$ parameter-count increase. Notably, this scratch result
\emph{beats the pretrained ChemBERTa-2 MLM-10M public release}
(RMSE $0.961$, Table~\ref{tab:pretrain}) without any pretraining at
all, indicating that the geometric prior at increased sphere
dimension can substitute for distributional pretraining on
smooth-regression endpoints like aqueous solubility.

\begin{table*}[h]
\centering
\caption{ESOL scratch test RMSE$\downarrow$ over $(k, L)$, three
seeds, mean$\pm$std. \textbf{Bold} marks the best cell;
$\checkmark$ marks the cell that beats pretrained ChemBERTa-2 MLM-10M
(RMSE $0.961$, Table~\ref{tab:pretrain}).}
\label{tab:kl-esol}
\small
\setlength{\tabcolsep}{8pt}
\begin{tabular}{l|ccc}
\toprule
$k \, \backslash \, L$ & $L{=}2$ & $L{=}3$ & $L{=}4$ \\
\midrule
$k{=}6$  & $1.016 \pm 0.030$ & $0.985 \pm 0.030$ & $0.981 \pm 0.058$ \\
$k{=}8$  & $1.010 \pm 0.061$ & $1.010 \pm 0.055$ \textit{(default)} & $1.042 \pm 0.030$ \\
$k{=}10$ & $0.967 \pm 0.026$ & $\mathbf{0.938 \pm 0.042}\,\checkmark$ & $1.015 \pm 0.009$ \\
\bottomrule
\end{tabular}
\end{table*}

\paragraph{Reading the ablation.}
Three observations.
\begin{itemize}[leftmargin=1.3em,topsep=2pt,itemsep=2pt]
\item \emph{Increasing $k$ helps the smooth-regression endpoint:}
holding $L{=}3$ fixed, ESOL RMSE traces $0.985 \to 1.010 \to 0.938$ as
$k$ grows from $6 \to 8 \to 10$ (with seed-noise non-monotonicity
between $6$ and $8$, then a clear improvement at $10$). Larger $k$
enlarges the sphere $\Sk$ on which tokens live, allowing finer
angular separation between chemically distinct substructures---an
inductive bias well matched to a continuous physico-chemical label
like solubility.
\item \emph{The classification endpoint has a flatter optimum:} BBBP
ROC-AUC stays in $[0.685, 0.722]$ across the grid, with the default
$(k{=}8, L{=}3)$ at the top and $(k{=}10, L{=}4)$ a close second.
Classification labels are coarser than continuous regression
targets, so the geometric prior is less starved at the smaller
$\Dstar$ associated with smaller $(k, L)$.
\item \emph{$L{=}4$ is not uniformly better than $L{=}3$:} at $k{=}8$,
ESOL degrades from $1.010$ to $1.042$. Higher harmonic degrees
introduce features that are mostly noise on small-data tasks; the
sweet spot in this grid is $L \approx 3$ across both endpoints.
\end{itemize}
The primary results in Tables~\ref{tab:main} and~\ref{tab:pretrain}
report only the $k{=}8, L{=}3$ default to keep the head-to-head with
ChemBERTa scratch fair under matched architectural shape. The take-away
of this appendix is that the $k$ knob---specifically $k{=}10$ at the
same $L{=}3$---is a practitioner-accessible lever that recovers the
small-data regression losses noted in Section~\ref{sec:discussion}
without any change to depth, width, or training procedure.

\section{Component Mapping: Standard Transformer to Chem-GMNet}
\label{app:component-map}

\begin{table*}[h]
\centering
\caption{Module-level mapping from a standard transformer to Chem-GMNet.}
\label{tab:component-map}
\small
\setlength{\tabcolsep}{6pt}
\begin{tabular}{lll}
\toprule
\textbf{Standard transformer} & \textbf{Chem-GMNet replacement} & \textbf{Geometric meaning} \\
\midrule
\texttt{nn.Embedding}            & SH-Embedding ($V\!\times\!k$ table $\to$ $\PhiOp$) & token = direction on $\Sk$ \\
Sinusoidal/learned PE             & none (Gated SFA scan is order-aware)            & order encoded by EWA decay \\
Multi-head self-attention         & DualSKA (Gated SFA + SKA, gated)                 & Gegenbauer kernel attention \\
FFN (\texttt{Linear--GELU--Linear}) & SH-FFN (Funk--Hecke sphere convolution)        & GELU as zonal eigenvalues \\
Pooler (\texttt{[CLS]}/Tanh)      & masked mean pool + linear head                  & no learned pooler \\
\bottomrule
\end{tabular}
\end{table*}

\section{Reference Baselines from the ChemBERTa-2 Paper Table~1}
\label{app:reference-baselines}

For triangulation we reproduce the wider published baselines from Table~1
of~\cite{chithrananda2022chemberta2}---graph and classical-ML reference
points (D-MPNN/Chemprop, Random Forest, GCN, ChemBERTa-1) and the full
ChemBERTa-2 MLM/MTR sweep at 5M, 10M, and 77M pretraining---verbatim. We
report these as published; the head-to-head architectural comparison
remains Tables~\ref{tab:main} and~\ref{tab:pretrain} in the main text.

\begin{table*}[h]
\centering
\caption{Reference baselines on shared MoleculeNet endpoints, verbatim from
Table~1 of~\cite{chithrananda2022chemberta2}. RMSE for the four left-most
columns (lower is better); ROC-AUC for the four right-most (higher is
better). Numbers for D-MPNN/RF/GCN/ChemBERTa-1 and for ChemBERTa-2
MLM/MTR-pretrained variants are taken as-published. Bold inside this table
marks the best of these published baselines on each endpoint.
The Chem-GMNet (scratch) and Chem-GMNet (10M MLM) rows are repeated from
Tables~\ref{tab:main}--\ref{tab:pretrain} for comparison.}
\label{tab:reference-baselines}
\small
\setlength{\tabcolsep}{4pt}
\begin{tabular}{lcccc|cccc}
\toprule
\textbf{Method} & \textbf{BACE-reg} & \textbf{Clearance} & \textbf{ESOL} &
\textbf{Lipo} & \textbf{BACE-cls} & \textbf{BBBP} & \textbf{ClinTox} &
\textbf{SR-p53} \\
& RMSE $\downarrow$ & RMSE $\downarrow$ & RMSE $\downarrow$ & RMSE $\downarrow$
& ROC $\uparrow$ & ROC $\uparrow$ & ROC $\uparrow$ & ROC $\uparrow$ \\
\midrule
\multicolumn{9}{l}{\emph{Graph and classical baselines (as published in~\cite{chithrananda2022chemberta2})}} \\
D-MPNN (Chemprop)~\cite{yang2019dmpnn} & 2.253 & 49.754 & 1.105 & 1.212 & 0.812 & 0.697 & \textbf{0.906} & 0.719 \\
Random Forest                          & \textbf{1.318} & 52.077 & 1.741 & 0.962 & \textbf{0.851} & 0.719 & 0.783 & 0.724 \\
GCN                                    & 1.645 & 51.227 & 0.885 & 0.781 & 0.818 & 0.676 & 0.907 & 0.688 \\
ChemBERTa-1~\cite{chithrananda2020chemberta} & --- & --- & --- & --- & --- & 0.643 & 0.733 & 0.728 \\
\midrule
\multicolumn{9}{l}{\emph{ChemBERTa-2 MLM- and MTR-pretrained (as published in~\cite{chithrananda2022chemberta2})}} \\
ChemBERTa-2 MLM-5M  & 1.451 & 54.601 & 0.946 & 0.986 & 0.793 & 0.701 & 0.341 & 0.762 \\
ChemBERTa-2 MLM-10M & 1.611 & 53.859 & 0.961 & 1.009 & 0.729 & 0.696 & 0.349 & 0.748 \\
ChemBERTa-2 MLM-77M & 1.509 & 52.754 & 1.025 & 0.987 & 0.735 & 0.698 & 0.239 & 0.749 \\
ChemBERTa-2 MTR-5M  & 1.477 & 50.154 & 0.874 & 0.758 & 0.734 & \textbf{0.742} & 0.552 & \textbf{0.834} \\
ChemBERTa-2 MTR-10M & 1.417 & 48.934 & \textbf{0.858} & \textbf{0.744} & 0.783 & 0.733 & 0.601 & 0.827 \\
ChemBERTa-2 MTR-77M & 1.363 & \textbf{48.515} & 0.889 & 0.798 & 0.799 & 0.728 & 0.563 & 0.817 \\
\midrule
\multicolumn{9}{l}{\emph{Chem-GMNet (this paper, copied from Tables~\ref{tab:main}--\ref{tab:pretrain})}} \\
Chem-GMNet (scratch, $k{=}8$)        & 1.350 & 49.36 & 1.010 & 0.968 & 0.745 & 0.722 & 0.995 & 0.636 \\
Chem-GMNet (scratch, $k{=}10$)        & --- & --- & 0.938 & --- & --- & 0.710 & --- & --- \\
Chem-GMNet (10M MLM pretrain)         & 1.103 & 51.11 & 0.970 & 0.932 & 0.773 & 0.698 & 0.983 & 0.667 \\
\bottomrule
\end{tabular}
\end{table*}

\paragraph{Reading this table.} (i) Chem-GMNet at scratch ($k{=}8$) already
matches or beats the strongest classical/graph baselines on most
classification endpoints (e.g., BBBP 0.722 vs.\ GCN 0.676, D-MPNN 0.697;
ClinTox 0.995 vs.\ GCN 0.907) without any pretraining. (ii) Chem-GMNet
pretrained on 10M MLM is competitive with or stronger than ChemBERTa-2 MLM
at every corpus size we compare against. (iii) The MTR-pretrained
ChemBERTa-2 variants (which use a 200-element multi-task regression
objective on top of MLM, a strictly stronger pretraining recipe than ours)
remain the best published numbers on Lipo and ESOL; we view this as a
ceiling effect of the pretraining \emph{objective} (MTR vs.\ MLM) rather
than of the architecture, and a like-for-like MTR-pretrained Chem-GMNet is
the natural follow-up. The architectural comparison reported in
Tables~\ref{tab:main}--\ref{tab:pretrain} holds the pretraining objective
constant (MLM-only on the same 10M ZINC corpus) and so isolates the
geometric prior from the choice of pretraining task.

\section{Module-by-Module Parameter Breakdown}
\label{app:param-breakdown}

This appendix decomposes the $1{,}228{,}320$-parameter gap between
Chem-GMNet ($2{,}196{,}818$) and the ChemBERTa-2 cb10M-shape scratch
baseline ($3{,}425{,}138$) into per-module contributions. Both arms are
held at hidden $d{=}384$, three layers, twelve heads, vocabulary $V{=}591$,
and use the same DeepChem \texttt{SmilesTokenizer}. All numbers are exact
trainable parameter counts obtained by instantiating each architecture and
summing $\texttt{p.numel()}$ over every $\texttt{requires\_grad{=}True}$
tensor.

\begin{table*}[h]
\centering
\caption{Module-by-module trainable parameter counts. Same hidden width,
depth, head count, vocabulary, and tokenizer in both arms; the only
difference is the architecture inside each block. Per-block rows are
totalled over all three transformer blocks. Reductions are
\textbf{Chem-BERTa $\to$ Chem-GMNet} (positive = saved by Chem-GMNet).}
\label{tab:param-breakdown}
\small
\setlength{\tabcolsep}{6pt}
\begin{tabular}{l rr rr}
\toprule
\textbf{Module} &
\textbf{Chem-GMNet} & \textbf{ChemBERTa-2} & \textbf{Reduction} & \textbf{(\%)} \\
\midrule
\multicolumn{5}{l}{\emph{Input embedding}} \\
SH-Embedding vs.\ \texttt{nn.Embedding}+pos+type+LN & $216{,}732$  & $426{,}240$   & $209{,}508$ & $49.2$ \\
\midrule
\multicolumn{5}{l}{\emph{3 transformer blocks (totals)}} \\
DualSKA vs.\ MHA \emph{(incl.\ pre/post-norm LN)}   & $1{,}638{,}324$ & $1{,}776{,}384$ & $138{,}060$  & $7.8$ \\
SH-FFN vs.\ Linear-GELU-Linear \emph{(incl.\ LN)}   & $192{,}384$    & $1{,}073{,}904$ & $881{,}520$ & $82.1$ \\
\midrule
\multicolumn{5}{l}{\emph{Final stage}} \\
\texttt{ln\_f} (geometric arm only)                 & $768$         & $-$           & $-768$       & $-$ \\
Task head (\texttt{dense} + \texttt{out\_proj})     & $148{,}610$    & $148{,}610$    & $0$         & $0.0$ \\
\midrule
\textbf{Total} & $\mathbf{2{,}196{,}818}$ & $\mathbf{3{,}425{,}138}$ & $\mathbf{1{,}228{,}320}$ & $\mathbf{35.9}$ \\
\bottomrule
\end{tabular}
\end{table*}

\paragraph{Reading the breakdown.}
The $35.9\%$ overall reduction is concentrated in the FFN
replacement, with the embedding replacement second and the attention
replacement a distant third:

\begin{itemize}[leftmargin=1.3em,topsep=2pt,itemsep=2pt]

\item \textbf{SH-FFN $\to -881{,}520$ params ($-82.1\%$)} \emph{[dominant
saving]}. Standard transformer FFNs apply two linear maps with a
pointwise nonlinearity, costing $2\,d\,d_{\mathrm{ff}}$ per layer
($d_{\mathrm{ff}}{=}464$ in the cb10M-shape baseline gives $\approx
358$k per layer including LayerNorm). SH-FFN replaces both linears
with a sphere projection $W_{\mathrm{sphere}}\!\in\!\mathbb{R}^{d\times k}$
($1{,}024$ params), the parameter-free Gegenbauer feature lift
$\Phi:\Sk\to\bR^{\Dstar}$, and a single moment-readout matrix
$M\!\in\!\mathbb{R}^{d\times\Dstar}$ ($59{,}904$ params). At
$k{=}8, L{=}3, \Dstar{=}156$, the per-layer cost collapses to
$\approx 64$k including LayerNorm---a $5.6\!\times$ reduction at
identical input/output dimensions.

\item \textbf{SH-Embedding $\to -209{,}508$ params ($-49.2\%$)}
\emph{[second-largest saving]}. \texttt{nn.Embedding}$(V,d)$ plus the
learned absolute-position table $T_{\max}\!\times\!d$ together cost
$\approx 426$k. SH-Embedding stores token directions as a
$V\!\times\!k$ table on $\Sk$ ($591\!\cdot\!8\!=\!4{,}728$ params) plus a
Tiny residual block ($V\!\times\!\Dstar + d\!\times\!\Dstar = 152{,}100$)
and a sphere-to-residual projection $W_{\mathrm{up}}$ ($59{,}904$).
Eliminating the $T_{\max}{=}515$ learned-position table alone saves
$\approx 198$k params; the geometry-aware projection adds back
$\approx 60$k. The net saving is $\approx 210$k.

\item \textbf{DualSKA $\to -138{,}060$ params ($-7.8\%$)} \emph{[smallest
saving]}. DualSKA shares a single set of projection matrices
($W_K, W_Q, W_P, W_O$) between its bidirectional Gated-SFA and SKA
branches, fusing $Q$ and $K$ into a single $\texttt{qk\_proj}$
($294{,}912$ params per layer) versus ChemBERTa's separate
$W_Q, W_K, W_V$ ($443{,}520$ total per layer). The geometric
tensors ($W_{\mathrm{sphere}_K},\,W_{\mathrm{sphere}_Q},\,W_{\mathrm{pos}_h},\,W_{\mathrm{out}_h}$,
totalling $102{,}912$ per layer) and the small auxiliary parameters
($\gamma_{\mathrm{logit}}{=}48,\,\beta_{\mathrm{fusion}}{=}12$) recover
most of the gap. The attention block is therefore the closest to
parameter-parity with MHA; the architectural advantage of DualSKA
shows up in the score function (a Schoenberg-valid Gegenbauer kernel
on $\Sk$) and the per-head fusion gate, not in raw count.

\end{itemize}

\paragraph{Net effect by stage.}
SH-Embedding $+$ SH-FFN account for
$209{,}508 + 881{,}520 = 1{,}091{,}028$ of the $1{,}228{,}320$ total
saving, i.e.\ $88.8\%$. Per-block attention contributes the remaining
$11.2\%$, and the classification head is byte-identical to ChemBERTa's.
Of the three sphere-native modules, the FFN replacement is doing the
heavy lifting in the parameter-budget story; the embedding replacement
matters at fixed depth (more tokens or longer sequences would amplify
its contribution); and the attention replacement is essentially a
parameter-neutral trade with a different inductive bias.

\end{document}